\documentclass[11pt]{article}
\usepackage{amsmath, amssymb, fullpage, ulem}

\newtheorem{theorem}{Theorem}[section]
\newtheorem{prop}[theorem]{Proposition}
\newtheorem{lemma}[theorem]{Lemma}

\newtheorem{corollary}[theorem]{Corollary}
\newtheorem{definition}[theorem]{Definition}
\newenvironment{proof}{\noindent{\bf Proof.}}{\hfill$\square$\medskip}
\def\eps{\epsilon}
\def\Pr{{\sf Pr}}
\def\prob{\Pr}

\def\max{\mbox{Max}}
\def\Min{\mbox{Min}}
\def\var{\mbox{Var}}
\def\R{{\mathbb R}}
\def\E{{\sf E}}
\def\Var{\mbox{Var}}
\def\D{{\cal D}}

\def\P{{\cal P}}

\newcommand{\cut}[1]{}

\begin{document}

\title{Chi-Squared Amplification: Identifying Hidden Hubs}
\author{Ravi Kannan\thanks{Microsoft Research India. Email: {\tt kannan@microsoft.com}}
\and
Santosh Vempala\thanks{Georgia Tech. Email: {\tt vempala@gatech.edu}}
}

\maketitle

\begin{abstract}
We consider the following general {\it hidden hubs} model: an $n \times n$ random matrix $A$ with a subset $S$ of $k$ special rows (hubs): entries in rows outside $S$ are generated from the 
(Gaussian) probability distribution $p_0 \sim N(0,\sigma_0^2)$; for each row in $S$, some $k$ of its entries are generated from $p_1 \sim N(0,\sigma_1^2)$, $\sigma_1>\sigma_0$, and the rest of the entries from $p_0$. The special rows with higher variance entries can be viewed as hidden higher-degree hubs. The problem we address is to identify them efficiently. This model includes and significantly generalizes the planted Gaussian Submatrix Model, where the special entries are all in a $k \times k$ submatrix.
There are two well-known barriers: if $k\geq c\sqrt{n\ln n}$, just the row sums are sufficient to find $S$ in the general model. 
For the submatrix problem, this can be improved by a $\sqrt{\ln n}$ factor to $k \ge c\sqrt{n}$ by spectral methods or combinatorial methods.
In the variant with $p_0=\pm 1$ (with probability $1/2$ each) and $p_1\equiv 1$, neither barrier has been broken (in spite of much effort, particularly for the submatrix version, which is called the Planted Clique problem.)

Here, we break both these barriers for the general model with Gaussian entries. 
We give a polynomial-time algorithm to
identify all the hidden hubs with high probability for $k \ge n^{0.5-\delta}$ for some $\delta >0$, when $\sigma_1^2>2\sigma_0^2$.  The algorithm extends easily to the setting where planted entries might have different variances each at least as large as $\sigma_1^2$.
We also show a nearly matching lower bound: for $\sigma_1^2 \le 2\sigma_0^2$, there is no
polynomial-time Statistical Query algorithm for distinguishing between a matrix whose entries are all from $N(0,\sigma_0^2)$ and a matrix with $k=n^{0.5-\delta}$ hidden hubs for any $\delta >0$. The lower bound as well as the algorithm are related to whether the chi-squared distance of the two distributions diverges. At the critical value $\sigma_1^2=2\sigma_0^2$, we show that the general hidden hubs problem can be solved for $k\geq c\sqrt n(\ln n)^{1/4}$, improving on the naive row sum-based method.
\end{abstract}

\thispagestyle{empty}
\newpage
\setcounter{page}{1}

\section{Introduction}

Identifying hidden structure in random graphs and matrices is a fundamental topic in unsupervised machine learning, with many application areas and deep connections to probability, information theory, linear algebra, statistical physics and other disciplines.
A prototypical example is finding a large hidden clique in a random graph, where the best known algorithms can find a clique of size $k=\Omega(\sqrt{n})$ planted in $G_{n,\frac{1}{2}}$, and smaller planted cliques are impossible to find by statistical algorithms \cite{FGRVX13} or using powerful convex programming hierarchies \cite{Barak16}.
A well-known extension to real-valued entries is the Gaussian hidden submatrix: each entry is drawn from  $N(0,\sigma_0^2)$, except for entries from a $k \times k$ submatrix, which are drawn from $N(\mu,\sigma_1^2)$.

Algorithms for both are based on spectral or combinatorial methods.
Information-theoretically, even a planting of size $O(\log n)$ can be found in time $n^{O(\log n)}$ by enumerating subsets of size $O(\log n)$. This raises the question of the threshold for efficient algorithms.
Since the planted part has different variance, it is natural to try to detect the planting using either the
sums of the rows (degrees in the case of graphs) or the spectrum of the matrix. 
However, these approaches can only detect the planting
at rather large separations (when $\mu=\omega(\sigma_0)$ for example)
or for $k = \Omega(\sqrt{n})$ \cite{Bop87, Kucera95,AKS98,Feige10, DekelGP11, BhaskaraCCFV10, MontanariRZ15, DeshpandeM15}.
Roughly speaking, the relatively few entries of the planted part must be large enough to dominate the variance of the many entries of the rest of the matrix.
A precise threshold for a rank-one perturbation to a random matrix to be noticeable was given by F\'eral and Pech\'e \cite{FeralP07} and applied in a
lower bound by Montanari et al. on using the spectrum to detect a planting \cite{MontanariRZ15}. Tensor optimization (or higher moment optimization)
rather than eigen/singular vectors can find smaller cliques \cite{FriezeK08,BrubakerV09}, but the technique has not yielded a polynomial-time algorithm to date.
A different approach to planted clique and planted Gaussian submatrix problems is to use convex programming relaxations, which also seem unable to go below $\sqrt{n}$.
Many recent papers demonstrate the limitations of these approaches \cite{FeigeK00, FGRVX13, MekaPW15, HopkinsKPRS16, Barak16,  FGV17} (see also \cite{Jerrum92}).

\paragraph{Model.} In this paper, we consider a more general model of hidden structure: the presence of a small number of hidden {\it hubs}. These hubs might represent more influential or 
atypical nodes of a network. Recovering such nodes is of interest in many areas (information networks, protein interaction networks, cortical networks etc.). In this model, as before, the entries of the matrix are drawn from $N(0,\sigma_0^2)$ except for special entries that all lie in $k$ rows, with $k$ entries from each of these $k$ rows. This is a substantial generalization of the above hidden submatrix problems, as the only structure is the existence of $k$ higher ``degree" rows (hubs) rather than a large submatrix. (Our results also extend to unequal variances for the special entries and varying numbers of them for each hub.) 

More precisely, we are given an $N\times n$ random matrix $A$ with independent entries.
There is some unknown subset $S$ of special rows, with $|S|=s$. Each row in $S$ has
$k$ special entries, each picked according to
$$p_1(x)\sim N(0,\sigma_1^2),$$
whereas, all the other $Nn-k|S|$ entries are distributed
according to $$p_0\sim N(0,\sigma_0^2).$$
The task is to find $S$, given, $s=|S|$, $k,n,\sigma_0^2,\sigma_1^2$.
One may also think of $S$ rows as picking $n$ i.i.d. samples from a mixture
	$$\frac{k}{n}p_1(x)+\left( 1-\frac{k}{n}\right)p_0(x),     $$
	whereas, the non-$S$ rows are picking i.i.d. samples from
	$p_0(x)$.
This makes it clear that we cannot assume that the planted entries in the
$S$ rows are all in the same columns.

If $\sigma_0^2=\sigma_1^2$, obviously, we cannot find $S$. If
$$\sigma_1^2>\sigma_0^2(1+c),$$
for a positive constant $c$ (independent of $n,k$), then it is easy to see
that $k\geq \Omega\left( \sqrt {n\ln n}\right)$ suffices to have a polynomial time
algorithm to find $S$: Set $B_{ij}=A_{ij}^2-1$. Let $\sum_jB_{ij}=\rho_i$.
It is not difficult to show that 
if $k\geq c\sqrt n\sqrt{\ln n}$, then, whp,
$$\mbox{Min}_{i: \mbox{ hub }}\;\; \rho_i> 2\mbox{Max}_{i:\mbox{ non-hub}}\;\; \rho_i.$$
The above algorithm is just the analog of the ``degree algorithm'' for hidden (Gaussian) clique --- take the $k$ vertices with the highest degrees --- and works with high probability for $k\geq c\sqrt{n\ln n}$. The remaining literature on upper bounds removes the $\sqrt{\ln n}$ factor, by using either a spectral approach (SVD) or a combinatorial approach (iteratively remove the minimum degree vertex).
For the general hub model, however, this improvement is not possible. The algorithms (both spectral and combinatorial) rely on the special entries being in a submatrix. This leads to our first question:\\

{\it Q. Are there efficient algorithms for finding hidden hubs for $k=o(\sqrt{n\ln n})$?}

\paragraph{Main results.}
Our main results can be summarized as follows. (For this statement, assume $\varepsilon,\delta$ are positive constants. In detailed statements later in the paper, they are allowed to depend on $n$.)
\begin{theorem}\label{thm:main-informal}
For the hidden hubs model with $k$ hubs:
\begin{enumerate}
\item For $\sigma_1^2=2(1+\varepsilon)\sigma_0^2 $, there is an efficient algorithm for $k \ge n^{0.5-\delta}$ for some $\delta>0$, depending only on $\varepsilon$.
\item For $\sigma_1^2 \in [c\sigma_0^2, 2\sigma_0^2]$, any $c >0$, no polynomial Statistical Query algorithm can detect hidden hubs for $k=n^{0.5-\delta}$, for any $\delta>0$.
\item At the critical value $\sigma_1^2=2\sigma_0^2$, with $N=n$, $k\geq \sqrt {n} \; (\ln n)^{1/4}$ suffices.
\end{enumerate}
\end{theorem}

Our algorithm also gives improvements for the special case of identifying hidden Gaussian cliques. For that problem, the closest upper bound in the literature is the algorithm of \cite{BhaskaraCCFV10} for detecting dense subgraphs. Their techniques could be used together with thresholding for distinguishing a hidden Gaussian clique instance from one with no planting. However, the resulting running time grows roughly as
$n^{O(1/(\eps-2\delta))}$ for $\sigma_1^2 = 2(1+\eps)\sigma_0^2$, and $\eps$ must be $\Omega(1)$ to be polynomial-time. Moreover, as with all previous algorithms, it does not extend to the hidden hubs model and needs the special (higher variance) entries to span a $k \times k$ submatrix. In contrast, our simple algorithms run in time linear in the number of entries of the matrix for $\eps = \Omega(1/\log n)$.

Our upper bound can be extended to an even more general model, where each planted entry could have its own distribution $p_{ij} \sim N(0, \sigma_{ij}^2)$ with bounded $\sigma_{ij}^2$.
There is a set of rows $S$ that are hubs, with $|S|=k$. For each $i\in S$, now we
assume there is some subset $T_i$ of higher variance entries. The $|T_i|$ are not given and need not be equal. We assume that the special entries satisfy:
$$\sigma_{ij}^2\geq \sigma_1^2,\mbox{  where, } \sigma_1^2=2(1+\varepsilon )\sigma_0^2, \varepsilon >0.$$

\begin{theorem}\label{main-thm-informal}
Let  $\tau_i =\sum_{j\in T_i} n^{-\sigma_0^2/\sigma_{ij}^2}$.
Suppose, for all $i\in S$,
$$\tau_i\geq \frac{1}{\sqrt\varepsilon} c (\ln N) (\ln n)^{0.5},$$
 then there is a randomized algorithm to identify all of $S$ with high probability.
\end{theorem}

As a corollary, we get that if $|T_i|=k$ for all $i\in S$, all special entries satisfy $\sigma_{ij}^2=\sigma_1^2$, and
	$$k=n^{.5-\delta},\mbox{  with } \varepsilon \geq \frac{2\delta}{1-2\delta}+\frac{\ln\ln N}{\ln n}+\frac{\ln\ln n}{2\ln n},$$
	then we can identify all of $S$.

We also have a result for values of $\varepsilon \in \Omega(1/\ln n)$. See Theorem (\ref{NEW-main}).

\paragraph{Techniques.}
Our algorithm is based on a new technique to amplify the higher variance entries, which we illustrate next.
Let $$p_0(x) = \frac{1}{\sqrt{2\pi}\sigma_0} \exp \left( - \frac{x^2}{2\sigma_0^2}\right)\quad \; \quad p_1(x) = \frac{1}{\sqrt{2\pi}\sigma_1} \exp \left( - \frac{x^2}{2\sigma_1^2}\right)$$
be the two probability densities.
The central (intuitive) idea behind our algorithm is to construct another matrix $\hat A$ 
of ``likelihood ratios'', 
defined as
$$\hat A_{ij} = \frac{p_1(A_{ij})}{p_0(A_{ij})} -1.$$
Such a transformation was also described in the context of the planted clique problem \cite{DeshpandeM2015} (although it does not give an improvement for that problem).
At a high level, one computes the row sums of $\hat A$ and shows that the row sums of the $k$ rows of the planted part
are all higher than all the row sums of the non-planted part. First, note that
$$E_{p_0} (\hat A_{ij} ) =\int p_1-\int p_0 =0\; ;\; \var_{p_0}(\hat A_{ij}) = \int \left( \frac{p_1}{p_0}-1\right)^2p_0 =\int\frac{p_1^2}{p_0} -1 = \chi^2(p_1\|p_0),$$
the $\chi$-squared distance between the two distributions $p_0,p_1$. Also,
$$E_{p_1} \left( \frac{p_1}{p_0}-1\right) = \chi^2(p_1\|p_0).$$
 Intuitively, since the expected sum of row $i$, for any $i\notin S$ is 0,
we expect success if the expected row sum in each row of $S$ is greater than the standard deviation of the row sum in any row not in $S$ times a log factor,
namely, if
\begin{equation}\label{success}
	\sqrt{\chi^2(p_1 \| p_0)} \geq \Omega^*( \frac{\sqrt n}{k}) = \Omega^*(n^\delta).
\end{equation}
$$\mbox{Now, }\quad \chi^2(p_1\|p_0) =\int \frac{p_1^2}{p_0} - 1 =\frac{c\sigma_0}{\sigma_1^2} \int \exp \left( x^2 \left( \frac{1}{2\sigma_0^2}-\frac{1}{\sigma_1^2}\right)\right).$$
So, if $\sigma_1^2\geq 2\sigma_0^2$, then, clearly, $\chi^2(p_1\|p_0)$ is infinite and so intuitively, (\ref{success}) can be
made to hold. This is not a proof. Indeed substantial technical work is needed to make this succeed. The starting point of that is to truncate entries, so the integrals are finite. We also have to compute higher moments to ensure enough concentration to translate these intuitive statements into rigorous ones.

On the other hand, if $\sigma_1^2 < 2 \sigma_0^2$, then $\chi^2(p_1 \| p_0)$ is finite and indeed bounded by a constant independent of $k,\sqrt n$.
So (\ref{success}) does not hold. This shows that this line of approach will not yield an algorithm. Our lower bounds show that there is no
polynomial time Statistical Query algorithm at all when $\sigma_1^2\in (0, 2\sigma_0^2]$.

The algorithms are based on the following transformation to the input matrix:
truncate each entry of the matrix, i.e., set the $ij$'th entry to $\min\{M,A_{ij}\}$, then apply $\frac{p_1(\cdot )}{p_0(\cdot )}$ to it; then take row sums. The analysis needs nonstandard a concentration inequality via a careful estimation of higher moments; standard concentration inequalities like the H\"offding inequality are not
sufficient to deal with the fact that the absolute bound on $p_1/p_0$ is too large.

Our algorithms also apply directly to the following {\it distributional version} of the hidden hubs problem with essentially the same separation guarantees.
A hidden hubs distribution is a distribution over vectors $x \in \R^n$ defined by a subset $S \subset [n]$ and parameters $\mu, \sigma_1, \sigma_0$ as follows:
$x_i \sim N(0, \sigma_0^2)$ for $i \not\in S$, and for $i \in S$,
\[
x_i \sim \begin{cases}
 N(\mu, \sigma_1^2)  & \text{with probability }\frac{k}{n}\\
N(0,\sigma_0^2) &  \text{ with probability } 1-\frac{k}{n}.
\end{cases}
\]
The problem is to identify $S$.

For almost all known distributional problems\footnote{The only known exception where a nonstatistical algorithm solves a distributional problem efficiently is learning parities with no noise using Gaussian elimination.}, the best-known algorithms are {\it statistical} or can be made statistical, i.e., they only need to compute expectations of functions on random samples rather than requiring direct access to the samples.
This characterization of algorithms, introduced by Kearns \cite{Kearns93, Kearns98}, has been insightful in part because it is possible to prove lower bounds on the complexity of statistical query algorithms. For example, Feldman et al. \cite{FGRVX13} have shown that the bipartite planted clique problem cannot be solved efficiently by such algorithms when the clique size is $k \le n^{0.5-\delta}$ for any $\delta >0$. A statistical query algorithm can query the input distribution via a statistical oracle. Three natural oracles are STAT, VSTAT and $1$-STAT. Roughly speaking, STAT($\tau$) returns the expectation of any bounded function on a random sample to within additive tolerance $\tau$; VSTAT($t$) returns the expectation of a $0/1$-valued function to within error no more than the standard deviation of $t$ random samples; and $1$-STAT simply returns the value of a $0/1$ function on a random sample.

For the hidden hubs problem, our algorithmic results show that one can go below the $\sqrt{n}$ threshold on the number of hubs (size of clique for the special case of hidden Gaussian clique).
Under the conditions of the algorithmic bounds, for $\sigma_1^2 \ge 2(1+\eps)\sigma_0^2$, there is a $\delta > 0$ s.t., a planting can be detected using a single statistical query whose tolerance is at most the standard deviation of the average of $O(n/k)$ independent samples. We complement the algorithmic results with a lower bound on the separation between parameters that is {\it necessary} for statistical query algorithms to be efficient (Theorem \ref{thm:LB}). 
Our application of statistical query lower bounds to problems over continuous distributions might be of independent interest.
Our matching upper and lower bounds can be viewed in terms of a single function, namely the 
$\chi$-squared divergence of the planted Gaussian and the base Gaussian.

The model and results raise several interesting open questions, including: (1) Can the upper bounds be extended to more general distributions on the entries, assuming independent entries? (2) Does the $\chi$-squared divergence condition suffice for general distributions? (3) Can we recover $k=O(\sqrt{n})$ hidden hubs when $\sigma_1^2 = 2\sigma_0^2$? (our current upper bound is $k=\sqrt{n}(\ln n)^{1/4}$ and our lower bounds do not apply above $\sqrt{n}$)  (4) Are there reductions between planted clique problems with $1/-1$ entries and the hidden hubs problem addressed here?

\paragraph{Summary of algorithms.}
Our basic algorithm for all cases is the same:

Define an $M$ (which is $\sigma_0\sqrt{\ln n}(1+o(1))$.) 
The exact value of $M$ differs from case to case.
Define
matrix $B$ by $B_{ij}=\exp( \gamma \Min (x^2,M^2))$, where, $\gamma$ is always
$=\frac{1}{2\sigma_0^2}-\frac{1}{2\sigma_1^2}$. Then, we prove that 
with high probability, the maximum 
$|S|$ row sums of $B$ occur precisely in the $S$ rows.
However, the bounds are delicate and so we present the proofs in each case separately
(taking advantage of the no page limit rule). 

\section{$\sigma_1^2=2\sigma_0^2$}

In this section, we assume
$$\sigma_1^2=2\sigma_0^2\quad \mbox{ and } N=n.$$

$$\frac{p_1}{p_0}=ce^{\gamma x^2},$$ where, $\gamma>0$ is given by:
\begin{equation}\label{gamma}
	\gamma =\frac{1}{2\sigma_0^2}\; -\; \frac{1}{2\sigma_1^2} = \frac{1}{4\sigma_0^2}.
\end{equation}
Define $L,M$ by:
\begin{equation}\label{M-inequality}
	L = \sqrt{ 2 \left( \ln n-\ln\ln n\right)}\quad ;\quad M=L\sigma_0.
\end{equation}
\begin{equation}\label{Bij-definition}
B_{ij}=\exp\left( \gamma \mbox{Min} (M^2,A_{ij}^2)\right).
\end{equation}
\begin{theorem}\label{main-theorem}
	If $$k\geq c\sqrt n (\ln\;\; n\;  )^{1/4},$$ then with
high probability, the top $s$ row sums of the matrix $B$ occur precisely in the $S$ rows.
\end{theorem}

\begin{prop}\label{prop-l}
Suppose $X$ is a non-negative real-valued random variable and $l$ is a positive integer.
$$E\left( |X-E(X)|^l\right)\leq 2E(X^l).$$
\end{prop}
\begin{proof}
\begin{align*}
E\left( |X-E(X)|^l\right) &\leq \int_{x=0}^{E(X)} (EX)^l \prob (X=x)\; dx +\int_{x=E(X)}^\infty x^l\prob (X=x)\; dx\\
     &\leq (EX)^l+E(X^l)\leq 2E(X^l),
\end{align*}
the last, since, $E(X)\leq (E(X^l))^{1/l}$.	
\end{proof}

\subsection{Non-planted entries are small}

Let
\begin{equation}
\mu_0=E_{p_0}(B_{ij})=\frac{1}{\sqrt{2\pi}\sigma_0}\int_{-\infty}^\infty \exp\left( \gamma \mbox{Min} (M^2,x^2)\right)\exp(-x^2/2\sigma_0^2).
\end{equation}
\begin{align}\label{mu0}
\mu_0&\leq \frac{1}{\sqrt{2\pi}\sigma_0}\int_{-\infty}^\infty \exp(\gamma x^2)p_0(x)\; dx =\frac{1}{\sqrt{2\pi}\sigma_0}\int_{-\infty}^\infty\exp(-x^2/2\sigma_1^2)=\sqrt 2.
\end{align}

%Since $B_{ij}\geq 0$,  for $l$ positive even integer
%\begin{equation}\label{BijMinusMu}
%E_{p_0}\left( B_{ij}-\mu_0\right)^l \leq \frac{2}{\sqrt{2\pi}\sigma_0} \left[ \int_0^{\sqrt{\ln (\mu_0)/\gamma }}\mu_0^lp_0(x)\; dx + \int_{\sqrt{\ln (\mu_0)/\gamma }}^\infty B_{ij}^lp_0(x)dx\right]
%\leq \mu_0^l +E_{p_0} (B_{ij}^l)\leq 2E_{p_0}(B_{ij}^l)
%\end{equation}
%the last step using jenson;s inequality to argue that $\mu_0=E_{p_0}(B_{ij})\leq \left( E_{p_0}(B_{ij}^l)\right)^{1/l}$.
.
%Also, we have $\int_{0}^M\exp(\lambda x^2)\leq \int_0^M \exp(\lambda Mx) \leq \frac{1}{\lambda M}\exp(\lambda M^2)$.

\begin{align}\label{2-exp-Bij-2}
E_{p_0}((B_{ij}-\mu_0)^2)&\leq  E_{p_0}(B_{ij}^2)\nonumber\\
&\leq \frac{2}{\sqrt{2\pi}\sigma_0}\int_{0}^M \exp(2\gamma x^2)\exp(-x^2/2\sigma_0^2) + \frac{2\exp(2\gamma M^2)}{\sqrt{2\pi}\sigma_0}\int_M^\infty \frac{x}{M}\exp(-x^2/2\sigma_0^2)dx\nonumber\\
&\leq \frac{2}{\sigma_0}\int_0^M \; dx\;  +\; \frac{2\sigma_0}{M}\exp\left( M^2 \left( 2 \gamma -\frac{1}{2\sigma_0^2}\right)\right)\leq \quad cL.
\end{align}

For $l\geq 4$, even, we have $\gamma l-(1/2\sigma_0^2)>0$ and
using Proposition (\ref{prop-l}), we get
\begin{align}\label{2-exp-Bij-l}
E_{p_0}((B_{ij}-\mu_0)^l)&\leq 2 E_{p_0}(B_{ij}^l)\nonumber\\
&\leq \frac{4}{\sqrt{2\pi}\sigma_0}\int_{0}^M \exp(\gamma lx^2)\exp(-x^2/2\sigma_0^2) + \frac{4\exp(\gamma lM^2)}{\sqrt{2\pi}\sigma_0}\int_M^\infty \frac{x}{M}\exp(-x^2/2\sigma_0^2)dx\nonumber\\
&\leq \frac{2}{\sigma_0}\int_0^M\exp\left( Mx\left(\gamma l-\frac{1}{2\sigma_0^2}\right)\right)\; dx\; +\; \frac{2\sigma_0}{M}\exp\left( M^2 \left(\gamma l-\frac{1}{2\sigma_0^2}\right)\right)\nonumber\\
&\leq \frac{c}{L}\exp\left( \frac{L^2(l-2)}{4}\right),
\end{align}

We will use a concentration result from (\cite{Kannan09}, Theorem 2) which specialized to our case states
\begin{theorem}\label{kannan-conc}
If $X_,X_2,\ldots ,X_n$ are i.i.d. mean 0 random variables, for any even positive integer $m$, we have
$$E\left( \left(\sum_{j=1}^nX_j\right)^m\right)\leq (cm)^m\left[ \sum_{l=1}^{m/2} \frac{1}{l^2}
    \left( \frac{nE(X_1^{2l})}{m}\right)^{1/l}\right]^{m/2}.$$
\end{theorem}

With $X_j=B_{ij}-\mu_0$, in Theorem (\ref{kannan-conc}),
we plug in the bounds of (\ref{2-exp-Bij-2}) and (\ref{2-exp-Bij-l}) to get:
\begin{lemma}\label{2-mth-moment}
\begin{align*}
	&\forall m\mbox{  even, } m\leq c\ln n, &E_{p_0}\left( \sum_{j=1}^n\left( B_{ij}-\mu_0\right)\right)^m&\leq (cmnL)^{m/2}
\end{align*}
\end{lemma}
\begin{proof}
\begin{align*}
	&\forall m\mbox{  even, }&E_{p_0}\left( \sum_{j=1}^n\left( B_{ij}-\mu_0\right)\right)^m&\leq (cm)^m\left[\frac{nL}{m}+ \exp(L^2/2)\sum_{l=2}^{m/2}\frac{1}{l^2}
\left( \frac{n}{mL}\exp(-L^2/2)       \right)^{1/l}\right]^{m/2}.
\end{align*}
Now, it is easy to check that
$$\frac{cnL}{m}\geq \exp(L^2/2)\left( n\exp(-L^2/2)/(mL)\right)^{1/l}\forall l\geq 2.$$
Hence the Lemma folows, noting that $\sum_l(1/l^2)\leq c$. 
\end{proof}

\begin{lemma}\label{2-Bij-non-planted}
	Let $$t=c\sqrt n\; (\ln n)^{3/4}.$$
for $c$ a suitable constant. For $i\notin S$,
$$\prob\left( \left| \sum_{j=1}^n (B_{ij}-\mu_0)\right|\geq \; t\; \right)\leq \frac{1}{n^2}.$$
Thus, we have
$$\prob\left( \exists i\notin S: \left| \sum_{j=1}^n (B_{ij}-\mu_0)\right|\geq \; t\; \right)\leq \frac{1}{n}.$$
\end{lemma}
\begin{proof}
We use Markov's inequality on the random variable
$\left| \sum_{j=1}^n (B_{ij}-\mu_0)\right|^m$ and Lemma (\ref{2-mth-moment}) with $m$ set to $4\ln n$ to get
$$\prob\left( \left| \sum_{j=1}^n (B_{ij}-\mu_0)\right|\geq \; t\; \right)\leq e^{-m}\leq \frac{1}{n^2},$$
giving us the first inequality. The second follows by
union bound.
\end{proof}
\subsection{Planted Entries are large}

Now focus on $i\in S$. Let $T_i$ be the set of $k$ special entries in row $i$.
We will use arguments similar to (\ref{2-exp-Bij-l}) to prove an upper bound on the $l$ th moment of
$B_{ij}-\mu_1$ for planted entries and use that to prove that $\sum_{T_i}B_{ij}$ is concentrated about
its mean.

We first need to get a lower bound on $\mu_1=E_{p_1}(B_{ij})$:
\begin{align*}
	\mu_1&\geq \frac{c}{\sigma_1}\int_0^M e^{x^2/4\sigma_0^2}e^{-x^2/4\sigma_0^2}\; dx =\frac{c}{\sigma_1}\int_0^M dx =cL.
\end{align*}
Let $l\geq 2$ be an integer.
Using Proposition (\ref{prop-l}), we get
\begin{align}\label{2-exp-Bij-l-planted}
E_{p_1}((B_{ij}-\mu_1)^l)&\leq 2 E_{p_1}(B_{ij}^l)\nonumber\\
			 &\leq \frac{4}{\sqrt{2\pi}\sigma_1}\int_{0}^M  \exp(\gamma lx^2)\exp(-x^2/2\sigma_1^2) + \frac{4\exp(\gamma lM^2)}{\sqrt{2\pi}\sigma_1}\int_M^\infty \frac{x}{M}\exp(-x^2/2\sigma_1^2)dx\nonumber\\
&\leq \frac{2}{\sigma_1}\int_0^M\exp\left( Mx\left(\gamma l-\frac{1}{2\sigma_1^2}\right)\right)\; dx\; +\; \frac{2\sigma_1}{M}\exp\left( M^2 \left(\gamma l-\frac{1}{2\sigma_1^2}\right)\right)\nonumber\\
&\leq \frac{4}{\sigma_1 M(2\gamma -(1/2\sigma_1^2))}\exp\left( M^2\left(\gamma l-\frac{1}{2\sigma_1^2}\right)\right) \leq \frac{c}{L}\exp\left( \frac{  L^2(l-1)}{4}\right).
\end{align}
\begin{lemma}\label{2-Bij-planted}
	Let $t$ be as in Lemma (\ref{2-Bij-non-planted}).
	Let $$t_2= c\left( \ln n\exp(L^2/4) \; +\; \frac{\sqrt {k\ln n}}{\sqrt L}\exp ( L^2/8)\right).$$
	$$\prob\left(\exists i\in S: \sum_{j\in T_i} (B_{ij}-\mu_1)< \; -\; t_2\right)\leq \frac{1}{n}.$$
	$$\prob\left( \exists i\in S: \sum_{j=1}^n (B_{ij}-\mu_0) < 100 t\right)<\frac{1}{n}.$$
\end{lemma}
\begin{proof}
	First, fix attention on one $i\in S$.
	We use Theorem (\ref{kannan-conc}) with $X_j=B_{ij}-\mu_1$ for $j\in T_i$.
	We plug in (\ref{2-exp-Bij-l-planted}) for $E(X_j^{2l})$ to get, with $m=4\ln N$:
	\begin{align*}
		E\left( \sum_{j\in T_i}(B_{ij}-\mu_1)\right)^m &\leq (cm\exp (L^2/4) )^m\left[ \sum_{l=1}^{m/2} \frac{1}{l^2} \left( \frac{k}{mL} \exp(-L^2/4)\right)^{1/l}\right]^{m/2}\\
							      &\leq
		(cm\exp(L^2/4))^m\left(\frac{k}{mL}\exp (-\; L^2/4) +1\right)^{m/2},
\end{align*}
	the last using $x^{1/l}\leq x+1$ for all $x>0$.
Now, we get that for a single $i\in S$, probability that $\sum_{j\in T_i} (B_{ij}-\mu_1)< \; -\; t_2$ is at most $1/n^2$
by using Markov inequality on $\left| \sum_{j\in T_i} (B_{ij}-\mu_1)\right|^m$. We get the first statement of the Lemma by a union bound over all $i\in S$.

For the second statement
we have, using the same argument as in Lemma (\ref{2-Bij-non-planted}), with high probability,
\begin{align}\label{2-8000}
\forall i\in S, \sum_{j\notin T_i} (B_{ij}-\mu_0)&\geq \; -\;  t.
\end{align}
We now claim that
$$kL>100 (t+t_2).$$
From the definition of $t,t_2$, it suffices to prove the following three inequalities to show this:
$$kL> c\sqrt n(\ln n)^{3/4}\; ;\;  kL>c\ln n e^{L^2/4}\; ;\; kL >\frac{\sqrt {k\ln n}}{\sqrt L}e^{L^2/8}.$$
Each is proved by a straightforward (but tedious) calculation.

From the first assertion of the Lemma and (\ref{2-8000}), we now
get that with high probability:
$$\sum_{j=1}^n(B_{ij}-\mu_0)\geq k(\mu_1-\mu_0) -t_2  -t\geq 100 (t+t_2),$$
proving Lemma (\ref{2-Bij-planted}).

\end{proof}

\section{$\sigma_1^2>2\sigma_0^2$}

Recall that all planted entries are $N(0,\sigma_1^2)$. There are $k$ planted entries in each row of $S$.
Assume (only)
$\varepsilon >\frac{c}{\ln n}.$
Define:
$$M^2=2\sigma_0^2(\ln n-\ln\varepsilon -\ln\ln N-\frac{1}{2}\ln \ln n) \quad \mbox{ and } \quad B_{ij}=\exp\left( \gamma \mbox{Min} (M^2,A_{ij}^2)\right).$$

\begin{theorem}\label{NEW-main}
	If $\varepsilon >c/\ln n$ and
	$$k> (\varepsilon \ln N\sqrt{\ln n})^{1-\frac{1}{2(1+\varepsilon)}}n^{1/(2(1+\varepsilon ))},$$
then with high probability, the top $s$ row sums of $B$ occur precisely in the $S$ rows.
\end{theorem}
\begin{corollary}
	If $\varepsilon >c/\ln n$ and $k\in\Omega^* \left( n^{0.5-\frac{\varepsilon }{2(1+\varepsilon )}}\right)$, then, with
	high probability, the
	top $s$ row sums of $B$ occur precisely in the $S$ rows.
\end{corollary}

\subsection{Non-planted entries are small}

Let
\begin{equation}
\mu_0=E_{p_0}(B_{ij})=\frac{1}{\sqrt{2\pi}\sigma_0}\int_{-\infty}^\infty \exp\left( \gamma \mbox{Min} (M^2,x^2)\right)\exp(-x^2/2\sigma_0^2)
\leq \frac{1}{\sqrt{2\pi}\sigma_0}\int_{-\infty}^\infty\exp(-x^2/2\sigma_1^2)=\sqrt{2(1+\varepsilon)}.
\end{equation}

%Since $B_{ij}\geq 0$,  for $l$ positive even integer
%\begin{equation}\label{BijMinusMu}
%E_{p_0}\left( B_{ij}-\mu_0\right)^l \leq \frac{2}{\sqrt{2\pi}\sigma_0} \left[ \int_0^{\sqrt{\ln (\mu_0)/\gamma }}\mu_0^lp_0(x)\; dx + \int_{\sqrt{\ln (\mu_0)/\gamma }}^\infty B_{ij}^lp_0(x)dx\right]
%\leq \mu_0^l +E_{p_0} (B_{ij}^l)\leq 2E_{p_0}(B_{ij}^l)
%\end{equation}
%the last step using jenson;s inequality to argue that $\mu_0=E_{p_0}(B_{ij})\leq \left( E_{p_0}(B_{ij}^l)\right)^{1/l}$.

Let $l\geq 2$ be an integer.
We note that $\gamma l-(1/2\sigma_0^2) >0$ for $l\geq 2$.
%Also, we have $\int_{0}^M\exp(\lambda x^2)\leq \int_0^M \exp(\lambda Mx) \leq \frac{1}{\lambda M}\exp(\lambda M^2)$.
Using Proposition (\ref{prop-l}), we get (recall $i\notin S$)
\begin{align}\label{exp-Bij-l}
E_{p_0}((B_{ij}-\mu_0)^l)&\leq 2 E_{p_0}(B_{ij}^l)\nonumber\\
&\leq \frac{4}{\sqrt{2\pi}\sigma_0}\int_{0}^M \exp(\gamma lx^2)\exp(-x^2/2\sigma_0^2) + \frac{4\exp(\gamma lM^2)}{\sqrt{2\pi}\sigma_0}\int_M^\infty \frac{x}{M}\exp(-x^2/2\sigma_0^2)dx\nonumber\\
&\leq \frac{2}{\sigma_0}\int_0^M\exp\left( Mx\left(\gamma l-\frac{1}{2\sigma_0^2}\right)\right)\; dx\; +\; \frac{2\sigma_0}{M}\exp\left( M^2 \left(\gamma l-\frac{1}{2\sigma_0^2}\right)\right)\nonumber\\
&\leq \frac{c\sigma_0}{M\varepsilon }\exp\left( M^2\left(\gamma l-\frac{1}{2\sigma_0^2}\right)\right),
\end{align}
using $2\gamma -(1/2\sigma_0^2) = \frac{\varepsilon}{2\sigma_0^2(1+\varepsilon)}\geq \frac{\varepsilon}{4\sigma_0^2}$.

With $X_j=B_{ij}-\mu_0$, in Theorem (\ref{kannan-conc}),
we plug in the bounds of (\ref{exp-Bij-l}) to get:
\begin{lemma}\label{NEW-mth-moment}
\begin{align}\label{Bij-m}
	&\forall m\mbox{  even, }&E_{p_0}\left( \sum_{j=1}^n\left( B_{ij}-\mu_0\right)\right)^m&\leq (cm)^me^{\gamma m M^2}\left[ \sum_{l=1}^{m/2}\frac{1}{l^2}
\left( \frac{cn\sigma_0}{mM\varepsilon } \exp(-M^2/(2\sigma_0^2))      \right)^{1/l}\right]^{m/2}\implies \nonumber\\
&\mbox{ With }m=4\ln N, & E_{p_0}\left| \sum_{j=1}^n\left( B_{ij}-\mu_0\right)\right|^{m}&\leq \left( cm \exp (\gamma M^2)\right)^{m}\left( 1+\frac{cn\sigma_0}{mM\varepsilon }\exp( - M^2/2\sigma_0^2)\right)^{m/2}.
\end{align}
\end{lemma}
Here, the last inequality is because $x^{1/l}\leq x+1$ for all real $x$
and further $\sum_l (1/l^2)$ is a convergent series.

\begin{lemma}\label{NEW-non-planted}
	Let $$t=c(\ln N)\exp(\gamma M^2)\left( 1+\frac{\sqrt{cn\sigma_0}}{\sqrt{mM\varepsilon }}\exp\left(-\frac{M^2}{4\sigma_0^2}\right)\right),$$
for $c$ a suitable constant. For $i\notin S$,
$$\prob\left( \left| \sum_{j=1}^n (B_{ij}-\mu_0)\right|\geq \; t\; \right)\leq \frac{1}{N^2}.$$
Thus, we have
$$\prob\left( \exists i\notin S: \left| \sum_{j=1}^n (B_{ij}-\mu_0)\right|\geq \; t\; \right)\leq \frac{1}{N}.$$
\end{lemma}
\begin{proof}
We use Markov's inequality on the random variable
$\left| \sum_{j=1}^n (B_{ij}-\mu_0)\right|^m$ and (\ref{Bij-m}) with $m$ set to $4\ln N$ to get
$$\prob\left( \left| \sum_{j=1}^n (B_{ij}-\mu_0)\right|\geq \; t\; \right)\leq e^{-m}\leq \frac{1}{N^2},$$
giving us the first inequality. The second follows by
union bound.
\end{proof}

\subsection{Planted Entries are large}

Now focus on $i\in S$.
We will use arguments similar to (\ref{exp-Bij-l}) to prove an upper bound on the $l$ th moment of
$B_{ij}-\mu_{1}$
for planted entries and use that to prove that $\sum_{T_i}B_{ij}$ is concentrated about
its mean. Let $l\geq 2$ be an integer.
Using Proposition (\ref{prop-l}), we get
\begin{align*}
	E_{p_{1}}((B_{ij}-\mu_{1})^l)&\leq 2 E_{p_1}(B_{ij}^l)\nonumber\\
			 &\leq \frac{4}{\sqrt{2\pi}\sigma_{1}}\int_{0}^M \exp(\gamma lx^2)\exp(-x^2/2\sigma_{1}^2) + \frac{4\exp(\gamma lM^2)}{\sqrt{2\pi}\sigma_{1}}\int_M^\infty \frac{x}{M}\exp(-x^2/2\sigma_{1}^2)dx\nonumber\\
&\leq \frac{2}{\sigma_{1}}\int_0^M\exp\left( Mx\left(\gamma l-\frac{1}{2\sigma_{1}^2}\right)\right)\; dx\; +\; \frac{2\sigma_{1}}{M}\exp\left( M^2 \left(\gamma l-\frac{1}{2\sigma_{ij}^2}\right)\right)\nonumber\\
&\leq \frac{4}{\sigma_0 M(2\gamma -(1/2\sigma_{1}^2))}\exp\left( M^2\left(\gamma l-\frac{1}{2\sigma_{1}^2}\right)\right) \leq \frac{c\sigma_0}{M}\exp\left( M^2(\gamma l-(1/2\sigma_{1}^2))\right).
\end{align*}
Now, applying Theorem (\ref{kannan-conc}), we get:
\begin{align}\label{NEW-exp-Bij-l-planted}
	E_{p_1}\left( \sum_{j\in T_i}(B_{ij}-\mu_1)^m\right)&\leq (cm\exp(\gamma M^2))^m\left[ \sum_{l=1}^{m/2}\frac{1}{l^2}\left( \frac{k\sigma_0}{mM}\exp(-M^2/2\sigma_1^2)\right)^{1/l}\right]^{m/2}
\end{align}
\begin{lemma}\label{NEW-planted}
Let $$t_2=c\ln N\exp(\gamma M^2)\left[ 1+\frac{c\sqrt k}{\sqrt{\ln N}(\ln n)^{1/4}}\exp( -M^2/4\sigma_1^2)  \right] .$$
\begin{align*}
	\prob\left( \exists i\in S: \left| \sum_{j\in T_i}(B_{ij}-\mu_1)\right|\geq t_2\right) &\leq\frac{1}{N}\\
	\prob\left(\exists i\in S: \sum_{j=1}^n (B_{ij}-\mu_0)<50 t\right)\leq\frac{1}{N}.
\end{align*}
\end{lemma}
\begin{proof}
	The first statement of the Lemma follows from (\ref{NEW-exp-Bij-l-planted}) with $m=4\ln N$
	by applying Markov inequality to $|\sum_{j\in T_i}(B_{ij}-\mu_1)|$ and then
	union bound over all $i\in S$ (using $\sum_l\frac{1}{l^2}x^{1/l}\leq \sum_l(1/l^2)(1+x)\leq c(1+x)$.)
	
	For the second statement, we start with a lower bound on $\mu_1$,.
\begin{equation}\label{mu-1})
	\mu_1\geq \frac{c}{\sigma_1}\int_0^M\exp( \gamma x^2-x^2/2\sigma_1^2)\geq \frac{c\sigma_0}{\varepsilon M}\exp(\gamma M^2-(M^2/2\sigma_1^2)),
\end{equation}
the last using: for $\lambda>0$, $\int_0^Me^{\lambda x^2}\geq \int_{M-(1/\lambda M)}^M \exp( \lambda (M-(1/\lambda M))^2)dx\geq c\exp(\lambda M^2)/\lambda M$.
[Note: We also needed: $M \geq 1/\varepsilon M$ which holds because $M\in O(\sqrt{\ln n})$ and $\varepsilon >c/\ln n$.]
We assert that
$$k\mu_1> c t,t_2.$$
This is proved by checcking three inequalities:
\begin{align*}
	\frac{kc\sigma_0}{\varepsilon M}{\exp(\gamma M^2-(M^2/2\sigma_1^2))}&> c\ln N \exp(\gamma M^2)\\
	\frac{kc\sigma_0}{\varepsilon M}{\exp(\gamma M^2-(M^2/2\sigma_1^2))}&>c\ln N\exp(\gamma M^2)\frac{\sqrt{n\sigma_0}}{\sqrt{mM\varepsilon}}\exp(-M^2/4\sigma_0^2)\\
	\frac{kc\sigma_0}{\varepsilon M}{\exp(\gamma M^2-(M^2/2\sigma_1^2))}&> \frac{c\ln N\exp(\gamma M^2)\sqrt k}{(\ln N)^{1/2}(\ln n)^{1/4}}\exp(-M^2/4\sigma_1^2).
\end{align*}
These all hold as can be checked by doing simple calculations.

Now, we have
$$\sum_{j=1}^n(B_{ij}-\mu_0)=k(\mu_1-\mu_0)+\sum_{j\in T_i}(B_{ij}-\mu_1)+\sum_{j\notin T_i}(B_{ij}-\mu_0).$$
The last term is at least $-t$ with high probability (the proof is exactly as for the non-planted entries). The second term is at least $-t_2$
(whp). We have already shown that $\mu_0\leq \sqrt 2$ and that $k\mu_1> 100( t+t_2+\mu_0)$. This
proves the second statement of the Lemma.
\end{proof}

Lemmas (\ref{NEW-planted}) and (\ref{NEW-non-planted}) together prove Theorem (\ref{NEW-main}).

{\bf Noise Tolerance} This algorithm can tolerate (adversarial) noise which can perturb $\Omega^*(e^{1/2\varepsilon })$
(which is, for example, a power of $n$ when $\varepsilon = c/\ln n$) of the planted
entries in each row of $S$. Here is a sketch of the argument for this: Note that the crucial lower bound on planted
row sums in $B$ comes from the lower bound on $k\mu_1$, the expected row sum in $S$ rows. The lower bound of $L$ on $\mu_1$
involves the integral (\ref{mu-1}). It is easy to see that we only loose a constant factor if the integral is taken from 0
to $M-\frac{\sigma_0^2}{\varepsilon M}$ (instead of to $M$). Thus, corruption of all $x\in 
\left[ M-\frac{\sigma_0^2}{\varepsilon M}\; ,\; M\right]$ would only cost a constant factor. It is easy to see that 
(i) there are $\Omega^*(e^{1/2\varepsilon })$ points in this interval and (ii) these are the worst possible points to be
corrupted.

\section{Generalization to unequal variances of planted entries}

We assume the non-planted entries of an $N \times n$ matrix are drawn from $N(0,\sigma_0^2)$.
There is again a set $S$ of ``planted'' rows, with $|S|=k$. For each $i\in S$, now we
assume there is some subset $T_i$ of ``planted entries''. [But $|T_i|$ are
	 not equal and we are not given $|T_i|$.]
Planted entry $(i,j)$ has distribution $p_{ij}\sim N(0,\sigma_{ij}^2)$.
We assume
each planted $$\sigma_{ij}^2\geq \sigma_1^2,\mbox{  where, } \sigma_1^2=2(1+\varepsilon )\sigma_0^2, \varepsilon >0.$$
\begin{equation}\label{tau-i}
\mbox{ Let } \tau_i =\sum_{j\in T_i} n^{-\sigma_0^2/\sigma_{ij}^2}.
\end{equation}
\begin{equation}\label{gamma-2}
	\mbox{  Let   }
\gamma =\frac{1}{2\sigma_0^2}\; -\; \frac{1}{2\sigma_1^2}.
\end{equation}
Define $M$ by:
\begin{equation}\label{M-2}
M = \sqrt 2\sigma_0\sqrt{ \ln n}.
\end{equation}
\begin{equation}\label{Bij-definition-2}
B_{ij}=\exp\left( \gamma \mbox{Min} (M^2,A_{ij}^2)\right).
\end{equation}
\begin{theorem}\label{main-thm-general}
With the above notation, if, for all $i\in S$,
$$\tau_i\geq \frac{1}{\sqrt\varepsilon} c (\ln N) (\ln n)^{0.5},$$
 then,
with high probability, the set of $k$ rows of $B$ with the largest row sums is precisley $S$.
\end{theorem}

\begin{corollary}
	If $|T_i|=k$ for all $i\in S$ and all planted $\sigma_{ij}^2=\sigma_1^2$, and
	$$k=n^{.5-\delta},\mbox{  with } \varepsilon \geq \frac{2\delta}{1-2\delta}+\frac{\ln\ln N}{\ln n}+\frac{\ln\ln n}{2\ln n},$$
	then, with high probability, the largest $k$ row sums of $B$ occur in the $S$ rows.
\end{corollary}

The analysis for the non-planted entries is the same as before.

\subsection{Planted Entries are large}

Now focus on $i\in S$.
We will use arguments similar to (\ref{exp-Bij-l}) to prove an upper bound on the $l$ th moment of
$B_{ij}-\mu_{ij}$
($\mu_{ij}=E_{p_{ij}}(B_{ij})$)
for planted entries and use that to prove that $\sum_{T_i}B_{ij}$ is concentrated about
its mean. Let $l\geq 2$ be an integer.
Using Proposition (\ref{prop-l}), we get
\begin{align}\label{exp-Bij-l-planted}
	E_{p_{ij}}((B_{ij}-\mu_{ij})^l)&\leq 2 E_{p_1}(B_{ij}^l)\nonumber\\
			 &\leq \frac{4}{\sqrt{2\pi}\sigma_{ij}}\int_{0}^M \exp(\gamma lx^2)\exp(-x^2/2\sigma_{ij}^2) + \frac{4\exp(\gamma lM^2)}{\sqrt{2\pi}\sigma_{ij}}\int_M^\infty \frac{x}{M}\exp(-x^2/2\sigma_{ij}^2)dx\nonumber\\
&\leq \frac{2}{\sigma_{ij}}\int_0^M\exp\left( Mx\left(\gamma l-\frac{1}{2\sigma_{ij}^2}\right)\right)\; dx\; +\; \frac{2\sigma_{ij}}{M}\exp\left( M^2 \left(\gamma l-\frac{1}{2\sigma_{ij}^2}\right)\right)\nonumber\\
&\leq \frac{4}{\sigma_0 M(2\gamma -(1/2\sigma_{ij}^2))}\exp\left( M^2\left(\gamma l-\frac{1}{2\sigma_{ij}^2}\right)\right) \leq \frac{c\sigma_0}{M}\exp\left( M^2(\gamma l-(1/2\sigma_{ij}^2))\right).
\end{align}
\begin{lemma}\label{Bij-planted}
	For $i\in S$, let $t_i=c\ln N \exp(\gamma M^2) \left(1+\frac{\sqrt{\tau_i}}{\sqrt{\ln N}(\ln n)^{1/4}}\right)$.
	$$\prob\left(\exists i\in S: \sum_{j\in T_i} (B_{ij}-\mu_{ij})< \; -\; t_i\right)\leq \frac{1}{N}.$$
	$$\prob\left( \exists i\in S: \sum_{j=1}^n (B_{ij}-\mu_0) < 100 t\right)<\frac{1}{N}.$$
\end{lemma}
\begin{proof}
	First, fix attention on one $i\in S$.
	We use a more general version of Theorem (\ref{kannan-conc}) also from (\cite{Kannan09}):
	\begin{theorem}\label{kannan-conc-2}
If $X_,X_2,\ldots ,X_n$ are independent (not necessarily identical) mean 0 random variables, for any even positive integer $m$, we have
$$E\left( \left(\sum_{j=1}^nX_j\right)^m\right)\leq (cm)^m\left[ \sum_{l=1}^{m/2} \frac{1}{l^2}
\left( \sum_{j=1}^n  \frac{E(X_j^{2l})}{m}\right)^{1/l}\right]^{m/2}.$$
\end{theorem}
We apply this
with $X_j=B_{ij}-\mu_{ij}$ for $j\in T_i$.
	We plug in (\ref{exp-Bij-l-planted}) for $E(X_j^{2l})$ to get, with $m=4\ln N$:
	\begin{align*}
		E\left( \sum_{j\in T_i}(B_{ij}-\mu_{ij})\right)^m &\leq (cm\exp (\gamma M^2) )^m \left[ \sum_{l=1}^{m/2}\frac{1}{l^2}\left( \sum_{j\in T_i}\frac{1}{mM}\exp(-M^2/2\sigma_{ij}^2)\right)^{1/l}\right]^{m/2}\\
 &\leq (cm)^m\exp(\gamma mM^2)\left[ \sum_{l=1}^{m/2}\frac{1}{l^2} \left(\sum_{j\in T_i}\frac{1}{mM}n^{-\sigma_0^2/\sigma_{ij}^2}\right)^{1/l}\right]^{m/2}\\
		&\leq (cm)^m\exp(\gamma mM^2)\left( 1+\frac{\tau_i^{m/2} }{(mM)^{m/2}}\right),
	\end{align*}
	the last using $x^{1/l}\leq x+1$ for all $x>0$.

Now, with $m=4\ln N$, we get that for a single $i\in S$, probability that $\sum_{j\in T_i} (B_{ij}-\mu_1)< \; -\; t_i$ is at most $1/N^2$
by using Markov inequality on $\left| \sum_{j\in T_i} (B_{ij}-\mu_1)\right|^m$ (noting: $M\geq c\sqrt {\ln n}$). We get the first statement of the Lemma by a union bound over all $i\in S$.

For the second statement,
we first need to get a lower bound on $\mu_{ij}$:
\begin{align*}
	\mu_{ij}&\geq\int_{x=0}^M \frac{c}{\sigma_{ij}}\exp( \gamma x^2-x^2/2\sigma_{ij}^2)\; dx \geq \frac{c\sigma_0}{M}\exp( \gamma M^2-M^2/2\sigma_{ij}^2),
\end{align*}
the last using: for $\lambda>0$, $\int_0^Me^{\lambda x^2}\geq \int_{M-(1/\lambda M)}^M \exp( \lambda (M-(1/\lambda M))^2)dx\geq c\exp(\lambda M^2)/\lambda M$. So,
\begin{equation}\label{total-mu-ij}
	\sum_{j\in T_i}\mu_{ij}\geq \frac{c\sigma_0}{M}\exp(\gamma M^2)\tau_i.
\end{equation}

We have, using the same argument as in Lemma (\ref{NEW-non-planted}), with high probability,
\begin{align}\label{8000}
\forall i\in S, \sum_{j\notin T_i} (B_{ij}-\mu_0)&\geq \; -\;  t.
\end{align}
Thus, from (\ref{8000}), (\ref{total-mu-ij}) and the first assertion of the current Lemma,
\begin{align*}
	\sum_{j=1}^n (B_{ij}-\mu_0)&= \sum_{j\in T_i}(B_{ij}-\mu_{ij}) \; +\; \sum_{j\in T_i} (\mu_{ij}-\mu_0)\; +\; \sum_{j\notin T_i}(B_{ij}-\mu_{ij})\\
				      &\geq -t_i \; +\; \frac{c\sigma_0}{M}\exp(\gamma M^2) - t.
\end{align*}
We would like to assert the follwing inequalities, which together prove the second assertion of the Lemma.
\begin{align*}
\frac{c\sigma_0}{M}\exp(\gamma M^2)\tau_i & > c\ln N \exp(\gamma M^2)\\																									 &> c(\ln N)\exp(\gamma M^2)\frac{\sqrt{\tau_i}}{\sqrt{\ln N}(\ln n)^{1/4}}\\
&>c\ln N \exp(\gamma M^2)\left( \frac{\sqrt{cn\sigma_0}}{\sqrt{mM\varepsilon}}\exp(-M^2/4\sigma_0^2)\right).
\end{align*}
Each follows by a simple calculation.

\end{proof}

\section{Statistical algorithms and lower bounds}\label{sec:lb}

For problems over distributions, the input is a distribution which can typically be accessed via  a sampling oracle that provide iid samples from the unknown distribution. {\it Statistical} algorithms are a restricted class of algorithms that are only allowed to query functions of the distribution rather than directly access samples. We consider three types of statistical query oracles from the literature. Let $X$ be the domain over which distributions are defined (e.g., $\{-1,1\}^n$ or $\R^n$).
\begin{enumerate}
\item STAT($\tau$): For any bounded function $f: X \rightarrow [-1,1]$, and any $\tau \in [0,1]$, STAT($\tau$) returns a number $p \in [\E_D(f(x))-\tau, \E_D(f(x))+\tau]$.
\item VSTAT($t$): For any function $f:X\rightarrow \{0,1\}$, and any integer $t > 0$, VSTAT($t$) returns a number $p \in [\E_D(f(x))-\gamma, \E_D(f(x))+\gamma]$ where $\gamma = \max \left\{\frac{1}{t},\sqrt{\frac{\Var_D(f)}{t}}\right\}$. Note that in the second term, $\Var_D(f) = \E_D(f)(1-\E_D(f))$.
\item $1$-STAT: For any $f:X \rightarrow \{0,1\}$, returns $f(x)$ on a single random sample from $D$.
\end{enumerate}
The first oracle was defined by Kearns in his seminal paper \cite{Kearns93, Kearns98} showing a lower bound for learning parities using statistical queries and analyzed more generally by Blum et al. \cite{Blum+94}. The second oracle was introduced in \cite{FGRVX13} to get stronger lower bounds, including for the planted clique problem. For relationships between these oracles (and simulations of one by another), the reader is referred to \cite{FGRVX13, FPV15}.

Our algorithm for the hidden hubs problem can be made statistical. We focus on the detection problem ${\cal P}$: determine with probability at least $3/4$ whether the input distribution is $N(0,\sigma_0^2)$ for every entry with no planting, or if it is a hidden hubs instance, i.e., on a fixed $k$-subset of coordinates, the distribution is a mixture of $N(0,\sigma_0^2)$ and $N(\mu, \sigma_1^2)$ where the latter distribution is used with mixing weight $k/n$.  To get a statistical version of our algorithm ($p_1/p_0$), consider the following query function $f$:
For a random sample (column) $x$, truncate each entry, apply $p_1/p_0 -\mu_0$, add all the entries and output $1$ if the sum exceeds $t_0$; else output $0$.

By Lemmas \ref{NEW-non-planted} and \ref{Bij-planted}, with $T_0=100t$ and the threshold $t$ as in  Lemma \ref{NEW-non-planted}, we have the following consequence: if there is no planting, the probability that this query is $1$ is at most $1/N$, while if there is a planting it is one with probability at least $\frac{k}{n}(1-\frac{1}{N})$. Thus it suffices to approximate the expectation to within relative error $1/2$. To do this with VSTAT($t$), we set
$t = Cn/k$ for a large enough constant $C$. Thus, a planted Gaussian of size $n^{0.5-\delta}$ can be detected with a single query to VSTAT($O(n/k)$), provided $\sigma_1^2 \ge 2(1+\eps)\sigma_0^2$.

We will now prove that this upper bound is essentially tight. For $c\sigma_0^2 \le \sigma_1^2 \le 2\sigma_0^2$,  for any $c > 0$, and $k=n^{0.5-\delta}$ for any $\delta > 0$, any statistical algorithm that detects hidden hubs must have superpolynomial complexity. For the lower bounds we assume the planted entries are drawn from $N(\mu,\sigma_1^2)$. The cases of most interest are (a) $\mu=0$ and (b)$\sigma_1 = \sigma_2$. In both cases, the lower bounds will nearly match algorithmic upper bounds.

\begin{theorem}\label{thm:LB}
For a planting of size $k=n^{\frac{1}{2}-\delta}$,
\begin{enumerate}
\item For $\mu = 0$ and $c\sigma_0^2 \le \sigma_1^2 \le 2\sigma_0^2(1-\eps)$, any $c > 0$, any statistical algorithm that solves ${\cal P}$ with probability at least $3/4$ needs $n^{\Omega(\log n)}$ calls to VSTAT($n^{1+\delta}$).
\item For $\mu = 0$ and $\sigma_1^2 = 2\sigma_0^2$, any statistical algorithm that solves ${\cal P}$ with probability at least $3/4$ needs $n^{\Omega(\log n/\log\log n)}$ calls to VSTAT($n^{1+\delta}$).
\item For $\mu=0$ and $\sigma_1^2 \le (2+ o(\delta))\sigma_0^2$, any statistical algorithm that solves ${\cal P}$ with probability at least $3/4$ needs $n^{\omega(1)}$ calls to VSTAT($n^{1+\delta}$).
\item For $\sigma_1 = \sigma_0$, if $\mu^2 = o(\sigma^2\ln (\sqrt{n}/k))$, any statistical algorithm that solves ${\cal P}$ with probability at least $3/4$ needs $n^{\omega(1)}$ calls to VSTAT($n^{1+\delta}$).
\end{enumerate}
Moreover, the number of queries to $1$-STAT for any of the above settings is $\Omega(n^{1+\delta})$.
\end{theorem}

The proof of the theorem is based on the notion of {\it Statistical Dimension with Average Correlation} defined in \cite{FGRVX13}. It is a generalization of statistical dimension as defined by Blum et al. \cite{Blum+94} for learning problems. We first need to define the correlation of two distributions $A,B$ and a reference distribution $U$, all over a domain $X$,
\[
\rho_U(A,B) = \E_X\left(\left(\frac{A(x)}{U(x)}-1\right)\left(\frac{B(x)}{U(x)}-1\right)\right).
\]
The average correlation of a set of distributions $\D$ with respect to reference distribution $U$ is
\[
\rho_U(\D) = \frac{1}{|\D|^2} \sum_{A,B \in \D} \rho_U(A,B).
\]

\begin{definition}
For $\bar{\gamma} > 0$, domain $X$, a set of distributions ${\cal D}$ over $X$ and a reference distribution
$U$ over $X$ the {\it statistical dimension} of ${\cal D}$ relative to $U$ with average correlation $\bar{\gamma}$ is denoted by SDA(${\cal D}, U, \bar{\gamma}$) and defined to
be the largest integer $d$ such that for any subset ${\cal D'} \subset {\cal D}$,
$|{\cal D'}| > |{\cal D}|/d \Rightarrow \rho_U({\cal D'}) \le \bar{\gamma}$.
\end{definition}

The main application of this definition is captured in the following theorem.

\begin{theorem}\cite{FGRVX13}
For any decision problem $\cal P$ with reference distribution $U$, let $\D$ be a set of distributions such that
$d=SDA(\D, U, \bar{\gamma})$. Then any randomized algorithm that solves $\P$ with probability at least $\nu > \frac{1}{2}$ must make at least $(2\nu-1)d$ queries to $VSTAT(1/3\bar{\gamma})$. Moreover, any algorithm that solves $\P$ with probability at least $3/4$ needs $\Omega(1)\min\{d, \frac{1}{\bar{\gamma}}\}$ calls to $1$-STAT.
\end{theorem}

\subsection{Average correlation}

For two subsets $S,T$, each of size $k$, the correlation of their corresponding distributions $F_S, F_T$ is
\[
\rho(F_S, F_T) = \left\langle \frac{F_S(x)}{F(x)}-1, \frac{F_T(x)}{F(x)}-1\right\rangle_F=\E_F\left(\left(\frac{F_S(x)}{F(x)}-1\right)\left(\frac{F_T(x)}{F(x)}-1\right)\right)
\]
where $F$ is the distribution with no planting, i.e., $N(0,\sigma_0^2)^n$.
For proving the lower bound at the threshold $\sigma_1^2=2\sigma_0^2$, it will be useful to define
$\bar{F}_S$ as $F_S$ with each coordinate restricted to the interval $[-M, M]$. We will set $M=\sigma_1 \sqrt{C\ln k}$. As before, we focus on the range $\sigma_1^2 \in [c\sigma_0^2, (2+o(1))\sigma_0^2]$.

\begin{lemma}
For $\sigma_1^2 < 2\sigma_0^2$
\[
\rho(F_S, F_T) = \frac{k^2}{n^2}
\left(\left(\frac{\sigma_0^2}{\sigma_1\sqrt{2\sigma_0^2-\sigma_1^2}}\right)^{|S\cap T|}\exp\left(\frac{\mu^2}{2\sigma_0^2-\sigma_1^2}\cdot |S\cap T|\right)-1\right).
\]
For $\sigma_1^2 = 2\sigma_0^2$,
\[
\rho(\bar{F}_S, \bar{F}_T) \le \frac{k^2(C\ln k)^{|S \cap T|/2}}{n^2}.
\]
For $\sigma_1^2 = (2+\alpha)\sigma_0^2$ and $\alpha = o(1)$,
\[
\rho(\bar{F}_S, \bar{F}_T) \le \frac{k^2}{n^2} k^{C\alpha|S \cap T|/4}.
\]
\end{lemma}
\begin{proof}
\begin{align*}
\rho(F_S, F_T) &= \left\langle \frac{F_S(x)}{F(x)}-1, \frac{F_T(x)}{F(x)}-1\right\rangle_F\\
&= \int \frac{dF_S(x)\, dF_T(x)}{dF(x)} -1\\
&= \frac{k^2}{n^2}\left(\Pi_{i\in S \cap T} \frac{\sigma_0}{\sqrt{2\pi} \sigma_1^2}\int \exp\left(-\frac{(x_i - \mu)^2}{2\sigma_1^2} - \frac{(x_i - \mu)^2}{2\sigma_1^2} + \frac{x_i^2}{2\sigma_0^2}\right)   -1 \right)\\
&= \frac{k^2}{n^2}\left(\Pi_{i\in S \cap T} \frac{\sigma_0}{\sqrt{2\pi} \sigma_1^2}\int \exp\left(-x_i^2 \cdot \frac{2\sigma_0^2 - \sigma_1^2}{2\sigma_1^2\sigma_0^2} -\frac{2\mu^2 - 4x_i\mu}{2\sigma_1^2} \right)   -1 \right)\\
\end{align*}
Setting $z = \frac{\sigma_1\sigma_0}{\sqrt{2\sigma_0^2 - \sigma_1^2}}$,
%($z$ might be complex!),
\begin{align*}
\rho(F_S, F_T) &= \frac{k^2}{n^2}\left(\prod_{i\in S \cap T} \frac{\sigma_0}{\sqrt{2\pi} \sigma_1^2}\int \exp\left(-\frac{(x_i - 2\mu z^2/\sigma_1^2 )^2}{2z^2} + \mu^2 \left(\frac{2z^2}{\sigma_1^4} - \frac{1}{\sigma_1^2}\right) \right)   -1 \right).
\end{align*}
We note that if $z^2 \le 0$, then the integral diverges.
%i.e., no statistical lower bound is possible if $2\sigma_0^2 < \sigma_1^2$. (There should be an algorithm!)
Assuming that $z^2 > 0$.
\begin{align*}
\rho(F_S, F_T) &= \frac{k^2}{n^2}\left(\prod_{i\in S \cap T}\frac{\sigma_0}{\sqrt{2\pi} \sigma_1^2}\int \exp\left(-\frac{(x_i - 2\mu z^2/\sigma_1^2 )^2}{2z^2} + \mu^2 \left(\frac{2\sigma_0^2}{\sigma_1^2(2\sigma_0^2-\sigma_1^2)} - \frac{1}{\sigma_1^2}\right) \right)   -1 \right)\\
&= \frac{k^2}{n^2}\left(\exp\left(\frac{\mu^2 |S \cap T|}{2\sigma_0^2 - \sigma_1^2}\right)\prod_{i\in S \cap T}\frac{\sigma_0}{\sqrt{2\pi} \sigma_1^2}\int \exp\left(-\frac{(x_i - 2\mu z^2/\sigma_1^2 )^2}{2z^2} \right)   -1 \right)\\
&= \frac{k^2}{n^2}\left(\left(\exp\left(\frac{\mu^2}{2\sigma_0^2 - \sigma_1^2}\right)\frac{\sigma_0 z}{\sigma_1^2}\right)^{|S\cap T|}   -1 \right)\\
\\
&= \frac{k^2}{n^2}\left(\left(\frac{\sigma_0^2}{\sigma_1\sqrt{2\sigma_0^2-\sigma_1^2}}\exp\left(\frac{\mu^2}{2\sigma_0^2 - \sigma_1^2}\right)\right)^{|S\cap T|}   -1 \right)\\
\end{align*}
Note that $\sigma_0^2\geq \sigma_1\sqrt{2\sigma_0^2-\sigma_1^2}$, so the above bound is of the form $\alpha\beta^{|S\cap T|},$ where $\beta>1$.
For the second part, we have
\begin{align*}
	\rho(\bar{F}_S, \bar{F}_T) &\le \frac{k^2}{n^2}\left(\frac{\sigma_0}{\sqrt{2\pi}\sigma_1^2}\int_{-M}^M 1\, dx\right)^{|S\cap T|}\\
&\le \frac{k^2}{n^2}\left(\frac{C\ln k}{2}\right)^{|S\cap T|/2}.
\end{align*}
The last part is similar. With $\sigma_1^2 = (2+\alpha)\sigma_0^2$,
\begin{align*}
	\rho(\bar{F}_S, \bar{F}_T) &\le \frac{k^2}{n^2}\left(\frac{\sigma_0}{\sqrt{2\pi}\sigma_1^2}\int_{-M}^M e^{\frac{\alpha x^2}{2\sigma_1^2}}\, dx\right)^{|S\cap T|}\\
&\le \frac{k^2}{n^2}\left(k^{C\alpha/2}\right)^{|S\cap T|/2}.
\end{align*}
\end{proof}

\subsection{Statistical dimension of planted Gaussian}

\begin{lemma}
Let $\sigma_1^2 < 2\sigma_0^2$ and $D$ be set of distributions induced by every possible subset of $[n]$ of size $k$.
Assume $\rho(F_S, F_T) \le \alpha \beta^{|S \cap T|}$ for some $\beta > 1$. Then, for any subset $A \subset D$ with
\[
|A| \ge \frac{2 {n \choose k}}{\ell!(n/2k^2)^\ell},
\]
the average correlation of $A$ with any subset $S$ is at most
\[
\rho(A,S) = \frac{1}{|A|}\sum_{T \in A} \rho(F_T, F_S) \le 2\alpha\beta^\ell.
\]
\end{lemma}
\begin{proof}
This proof is similar to \cite{FGRVX13}. Define $T_r = \{T \in A\, : \, |T\cap S| = r\}$. Then,
\[
\sum_{T \in A} \rho(F_S, F_T) \le \alpha \sum_{T \in A} \beta^{|S\cap T|} = \alpha \sum_{r=r_0}^{k} |T_r \cap A|\beta^r.
\]
To maximize the bound, we would include in $A$ sets that intersect $S$ in $k-1$ indices, then $k-2$ indices and so on. Taking this extremal choice of $A$ gives us a lower bound on the minimum intersection size $r_0$ as follows. Note that for $0 \le j \le k-1$,
\begin{align*}
\frac{|T_{j+1}|}{|T_{j}|} &= \frac{{k \choose j+1}{n-k \choose k-j-1}}{{k \choose j}{n-k \choose k-j}}\\
&= \frac{(k-j)^2}{(j+1)(n-2k+j+1)}\\
&\le \frac{k^2}{jn}
\end{align*}
where the last step assumes $2k^2 < n$. Therefore,
\[
|T_j| \le \frac{1}{j!}\left(\frac{k^2}{n}\right)^j |T_0| \le \frac{{n \choose k}}{j! (n/k^2)^j}.
\]
This gives a bound on the minimum intersection size since
\[
\sum_{j=r_0}^k |T_j| < \frac{2{n \choose k}}{r_0!(n/k^2)^{r_0}}
\]
Therefore under the assumption on $|A|$, we get that $r_0 < \ell$. Using this,
\begin{align*}
\sum_{T \in A} \rho(F_S, F_T) &\le \alpha\sum_{r=r_0}^k |T_r \cap A| \beta^r\\
&\le \alpha\left(|T_{r_0} \cap A| \beta^{r_0} + \sum_{r= r_0+1}^k |T_r|\beta^r\right)\\
&\le \alpha\left(|T_{r_0}\cap A| \beta^{r_0}+2|T_{r_0+1}|\frac{\beta^{r_0+1}-1}{(r_0+1)(\beta-1)}\right)\\
&\le 2\alpha|A|\beta^{r_0+1} \le 2\alpha\beta^{\ell}|A|.
\end{align*}
\end{proof}

\begin{theorem}\label{thm:sd}
For the planted Gaussian problem ${\cal P}$, with
(a) $\sigma_1^2 < 2\sigma_0^2$, and average correlation at most
\[
\bar{\gamma} = 2\frac{k^2}{n^2}\left(\frac{\sigma_0^2}{\sigma_1\sqrt{2\sigma_0^2-\sigma_1^2}}\exp\left(\frac{\mu^2}{2\sigma_0^2 - \sigma_1^2}\right)\right)^{\ell}
\]
or (b) $\sigma_1^2=2\sigma_0^2$, and average correlation
\[
\bar{\gamma} =  2\frac{k^2}{n^2}\left(\frac{C\ln k}{2}\right)^{\ell/2}
\]
or
(c) $\sigma_1^2 =(2+\alpha)\sigma_0^2$ for $\alpha=o(1)$, and average correlation
\[
\bar{\gamma} = 2\frac{k^2}{n^2} k^{C\alpha \ell /4}
\]
the statistical dimension of ${\cal P}$ is at least $\ell! (n/k^2)^\ell /2$.
\end{theorem}

We now state explicitly the three main corollaries of this theorem. This completes the proof of Theorem \ref{thm:LB}.

\begin{corollary}\label{zero-mu}
With $\mu = 0$, and $\sigma_1^2 = 2\sigma_0^2(1-\eps)$, we have
\[
\bar{\gamma} = 2\frac{k^2}{n^2} \left(\frac{1}{4\eps(1-\eps)}\right)^{\ell/2}
\]
and for any $\delta > 0$, with $k = n^{0.5 - \delta}$, $\ell = c\log n/\log(1/\eps(1-\eps))$, we have $\bar{\gamma} = 2n^{c-2\delta-1}$ and
\[
SDA({\cal P}, \bar{\gamma}) = \Omega(n^{2\delta\log_{\frac{1}{\eps(1-\eps)}} n}).
\]
Hence with $c = \delta$, any statistical algorithm that solves ${\cal P}$ with probability at least $3/4$ needs $n^{\Omega(\log n)}$ calls to VSTAT($n^{1+\delta}$).
\end{corollary}

We note that the above corollary applies for any $0 < \sigma_1^2 < 2\sigma_0^2$, with the bounds depending mildly on how close $\sigma_1^2$ is to the ends of this range. This is quantified by the dependence on $\eps(1-\eps)$ above.

Our lower bound extends slightly above the threshold $\sigma_1^2=2\sigma_0^2$. For this, we need to observe that with respect to any $n^C$ samples, the distributions $F_S$ and $\hat{F}_S$ are indistinguishable with high probability ($1-n^{-C}$). Therefore, proving a lower bound on the statistical dimension of $\P$ with distributions $\hat{F}_S$ is effectively a lower bound for the original problem $\P$ with distributions $F_S$.

\begin{corollary}\label{lem:At-threshold}
With $\mu = 0$, and $\sigma_1^2 = 2\sigma_0^2$, we have
\[
\bar{\gamma} = 2\frac{k^2}{n^2} \left(\frac{C\ln k}{2}\right)^{\ell/2}
\]
and for any $\delta > 0$, with $k = n^{0.5 - \delta}$, $\ell = c\log n/2\log\log k$, we have $\bar{\gamma} = 2n^{c-2\delta-1}$ and
\[
SDA({\cal P}, \bar{\gamma}) = \Omega(n^{\delta\log n/\log\log n}).
\]
Hence with $c = \delta$, any statistical algorithm that solves ${\cal P}$ with probability at least $3/4$ needs $n^{\Omega(\log n/\log\log n)}$ calls to VSTAT($n^{1+\delta}$).
Moreover, for $\sigma_1^2 = (2+\alpha)\sigma_0^2$, $\alpha = o(\delta)$, we have
\[
\bar{\gamma} = 2\frac{k^2}{n^2} k^{C\alpha \ell /4}
\]
and for any $\delta > 0$, with $k = n^{0.5 - \delta}$, $\ell = 8\delta/C\alpha$, we have $\bar{\gamma} = 2n^{-\delta-1}$ and
\[
SDA({\cal P}, \bar{\gamma}) \ge n^{\delta\ell}.
\]
Hence any statistical algorithm that solves ${\cal P}$ with probability at least $3/4$ needs $n^{\omega(1)}$ calls to VSTAT($n^{1+\delta}$).
\end{corollary}

\begin{corollary}\label{equal-sigmas}
For $\sigma_1 = \sigma_0$,
\[
\bar{\gamma} = 2\frac{k^2}{n^2}\exp\left(\frac{\mu^2\ell}{\sigma^2}\right).
\]
and for any $\delta > 0$, with $k = n^{0.5 - \delta}$, $\mu^2 = c \sigma^2 \ln (\sqrt{n}/k)$, we have $\bar{\gamma} = 2n^{c\delta\ell -2\delta-1}$ and
\[
SDA({\cal P}, \bar{\gamma}) = \Omega(n^{2\delta\ell}).
\]
If $\mu^2 = o(\sigma^2\ln (\sqrt{n}/k))$, any statistical algorithm that solves ${\cal P}$ with probability at least $3/4$ needs $n^{\omega(1)}$ calls to VSTAT($n^{1+\delta}$).
\end{corollary}

\bibliographystyle{alpha}
\bibliography{planted}

\end{document}